\documentclass[letterpaper, 10 pt, conference]{ieeeconf}  

\IEEEoverridecommandlockouts                              
\usepackage{bmathshortcuts}

\usepackage{amsthm}

\newtheorem{theorem}{Theorem}

\newtheorem{corollary}[theorem]{Corollary}

\usepackage{algorithmic}
\usepackage[ruled,vlined]{algorithm2e}
\usepackage{amsmath,amssymb,amsfonts}
\usepackage{bm}
\usepackage{booktabs}
\usepackage{censor}
\usepackage{cite} 
\usepackage{color}
\usepackage[table]{xcolor}
\usepackage{comment}
\usepackage{epsfig}
\usepackage{fancyhdr}
\usepackage{float}
\usepackage{gensymb}
\usepackage{graphics}
\usepackage{graphicx}
\usepackage[hidelinks]{hyperref} 
\usepackage{mathptmx}
\usepackage{mathrsfs}
\usepackage{makecell}
\usepackage{soul}
\usepackage{textcomp}
\usepackage{tikz}
\usepackage{times}
\usepackage[dvipsnames]{xcolor} 
\usepackage{url}

\definecolor{lightblue}{RGB}{173,216,230}
\definecolor{lightpink}{RGB}{255, 204, 204}

{\vspace{-\topsep}\begin{itemize}\itemsep1pt \parskip0pt \parsep1pt}
{\end{itemize}\vspace{-\topsep}}

{\vspace{-\topsep}\begin{enumerate}\itemsep1pt \parskip0pt \parsep1pt}
{\end{enumerate}\vspace{-\topsep}}

\title{\LARGE \bf Trajectory Generation for Underactuated Soft Robot Manipulators using Discrete Elastic Rod Dynamics}

\newif\ifreview
\reviewfalse

\ifreview
\author{
    \censor{Beibei Liu}$^{1}$, \censor{Akua K. Dickson}$^{1}$, \censor{Ran Jing}$^{2}$, \censor{Andrew P. Sabelhaus}$^{1,2}$%
    \thanks{$^{1}$\censor{Beibei Liu, Akua K. Dickson and Andrew P. Sabelhaus} are with the \xblackout{Division of Systems Engineering, Boston University, Boston MA, USA}.
            \censor{{\tt\small \{liubb, akuad\}@bu.edu}}}%
    \thanks{$^{2}$\censor{Ran Jing and Andrew P. Sabelhaus} are with the \xblackout{Department of Mechanical Engineering, Boston University, Boston MA, USA}.
            \censor{{\tt\small \{rjing, asabelha\}@bu.edu}}}%
}
\else
\author{Beibei Liu$^{1}$, Akua K. Dickson$^{1}$, Ran Jing$^{2}$, Andrew P. Sabelhaus$^{1,2}$
\thanks{$^{1}$Beibei Liu, Akua K. Dickson and Andrew P. Sabelhaus are with the Division of Systems Engineering, Boston University, Boston MA, USA.
        {\tt\small \{liubb, akuad\}@bu.edu}}%
\thanks{$^{2}$ Ran Jing and Andrew P. Sabelhaus are with the Department of Mechanical Engineering, Boston University, Boston MA, USA.
        {\tt\small \{rjing, asabelha\}@bu.edu}}%
}
\fi

\begin{document}

\maketitle
\pagestyle{empty}  
\thispagestyle{empty} 

\begin{abstract}
Soft robots are well suited for contact-rich tasks due to their compliance, yet this property makes accurate and tractable modeling challenging. Planning motions with dynamically-feasible trajectories requires models that capture arbitrary deformations, remain computationally efficient, and are compatible with underactuation. However, existing approaches balance these properties unevenly: continuum rod models provide physical accuracy but are computationally demanding, while reduced-order approximations improve efficiency at the cost of modeling fidelity. To address this, our work introduces a control-oriented reformulation of Discrete Elastic Rod (DER) dynamics for soft robots, and a method to generate trajectories with these dynamics. The proposed formulation yields a control-affine representation while preserving certain first-principles force–deformation relationships. As a result, the generated trajectories are both dynamically feasible and consistent with the underlying actuation assumptions. We present our trajectory generation framework and validate it experimentally on a pneumatic soft robotic limb. Hardware results demonstrate consistently improved trajectory tracking performance over a constant-curvature-based baseline, particularly under complex actuation conditions.

\end{abstract}


\section{INTRODUCTION}
\label{Sect: Intro}
Soft robots have shown great potential for contact-rich and deformation-driven tasks, such as human-robot interaction and operation in delicate, unstructured environments \cite{vasicSafetyIssuesHumanrobot2013,Laschi2016Soft,
Sanan2011Physical, Dickson2025Safe}. 
Even though their inherent compliance improves safety \cite{majidi_2014_soft_robots} and flexibility \cite{Hsuan2019SoftMedical} within the environment, it also introduces challenges for modeling \cite{Das2019Review}. 
Task performance is enhanced when the model is computationally efficient, physically interpretable, and compatible with hardware implementation while deforming per arbitrary loading conditions.


\begin{figure}[t]
\centerline{\includegraphics[width=1\linewidth]{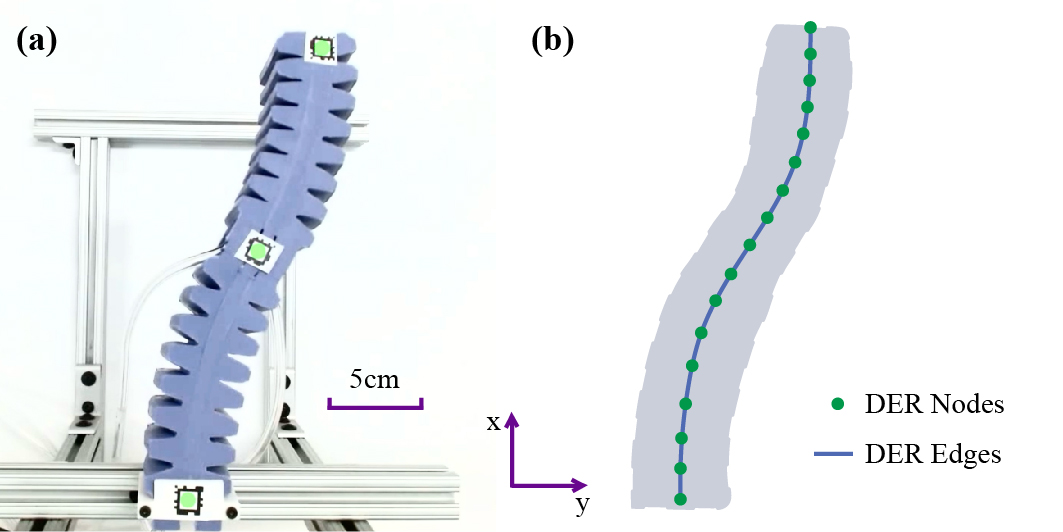}}
\caption{
This work transforms a physics-based model of interconnected elastic rods into a control-affine representation for underactuated soft robots by introducing an explicit actuator-to-force mapping. The soft limb (a) executes a motion profile generated by our resulting model to track a desired trajectory. The Discrete Elastic Rod (DER) model uses a set of point masses to approximate the rod's backbone, as illustrated in (b).
}
\label{Fig:Overview}
\vspace{-0.5cm}
\end{figure}

To date, no one dynamics model of soft robot manipulators has fully met each of these criteria.
High-fidelity models, such as Cosserat rod and strain-based formulations, accurately capture deformation \cite{doroudchi_configuration_2021,renda_dynamics_2024}, but require solving complex PDEs which limit real-time performance.
Reduced-order models with strong kinematic assumptions such as piecewise constant curvature (PCC) provide computational efficiency, yet may sacrifice direct correspondence to underlying continuum mechanics and may fail to capture important deformation modes \cite{webster2010design,katzschmann2019dynamic, della_santina_model-based_2020, della_santina_improved_2020}. 
Furthermore, soft robotic systems are typically underactuated \cite{Manti2015Underactuated}. 
In this setting, simplified dynamic models can fail to accurately predict deformations under both actuation and complex external loads, whereas highly-detailed models can lack a clear mapping from actuator states to the model's generalized forces \cite{Thuruthel2020First-Order, Naughton2021Elastica}. Consequently, developing a model that balances simplicity, physical interpretability, and real-time performance remains an open challenge.

An intermediate modeling representation could bridge this gap. Discrete Elastic Rod (DER) models naturally occupy this middle ground, as they discretize high-fidelity continuum models while preserving the underlying mechanics assumptions, therefore remaining computationally tractable \cite{Choi2024DisMech, Radha2025Py-DiSMech, Choi2021Implicit}. 
However, existing DER formulations are not fully compatible with underactuated soft robots \cite{Choi2021Implicit}. These methods either lack actuator models or rely on implicit mappings from actuation to deformation variables such as curvature.
And to the best of our knowledge, no prior work has generated dynamically-feasible state-input trajectories for task-space objectives with a discrete elastic rod.

To address this limitation, we propose a control-oriented reformulation of DER dynamics for soft robotic systems, which preserves its mechanics assumptions while enabling compatibility with underactuated control. 
Our approach explicitly incorporates actuation in a control-affine structure that enables direct computation of control inputs, avoiding the need for implicit or post-hoc actuation mappings. 
We show how to generate dynamically feasible trajectories using this property. 
Then, we execute the generated trajectories using the proposed model on a pneumatic soft robot in open-loop, where comparisons with PCC-based methods show improved performance under complex actuation conditions.

In this manuscript, we contribute:
\begin{itemize}
\item A reformulation of DER dynamics that explicitly incorporates actuation and yields a control-affine representation for underactuated soft robot manipulators (Fig.~\ref{Fig:Overview}),
\item A proof that the proposed reformulation preserves certain loading-deformation relationships consistent with classical beam mechanics, and
\item An experimental validation showing that our control sequence outperforms the PCC model in different task-space motions in hardware.
\end{itemize}

\section{Related Work}
\label{Sect: Related Work}
In this section, we review prior work on Discrete Elastic Rod (DER) models in soft robotics, focusing on the gap between their success in simulation and the challenges that arise when applying them to practical systems, and how this gap impacts trajectory generation.

\subsection{DER Modeling and Simulation}
DER models represent slender deformable structures using a finite set of discrete elements that capture bending, twisting, and stretching behaviors. 
Existing DER formulations are commonly constructed either from discretizations based on continuum rod partial differential equation (PDE) \cite{Gaston2026DERModel, Till2017Elastic, Gazzola2018softFilament, Zhang2019Elastica} or discretizations based on discrete differential geometry (DDG)~\cite{Choi2024DisMech}. DER often yields a compact state representation and dynamics\cite{Bergou2008Discrete, Goldberg2019OnPD, Liu2022Geometrically}, enabling them to capture deformation-driven mechanics beyond purely kinematic approximations. Therefore, DER frameworks are widely adopted for simulation. For example, DER-based models have been used to represent soft robot locomotion in both planar and spatial settings \cite{Goldberg2019OnPD}. Recently, DDG-based simulation platforms such as DisMech have provided solutions for highly dynamic, contact-rich scenarios \cite{Choi2024DisMech}, while lightweight implementations (e.g., MATLAB-based frameworks) create accessible environments for rapid prototyping and validation \cite{Goldberg2019OnPD, Gaston2026DERModel}. However, despite their success in simulation, extending DER models to practical robotic systems remains challenging and has received comparatively limited attention, particularly in the context of incorporating physically meaningful actuation modeling.

\subsection{Actuation and Input–State Modeling in DER Frameworks}
Most existing DER formulations are primarily \emph{simulation-oriented}, i.e., they are designed to reproduce deformation and contact robustly, while the mapping from physically realizable actuation to model inputs is often not explicitly defined \cite{Huang2022Design, Goldberg2019OnPD, Choi2021Implicit}. As a result, when DER models are adapted for practical soft robotic systems, researchers typically face a choice between two suboptimal approaches. Firstly, in many implementations, control inputs are defined through deformation variables (e.g., intrinsic curvature) \cite{Choi2024DisMech, huang_dynamic_2020}, which are then mapped to actuation through additional approximations or calibration. Conversely, if one considers direct force inputs as in the original DER physics, the system becomes inherently underactuated. In this setting, low-dimensional inputs must regulate high-dimensional deformations through complex, nonlinear geometric coupling~\cite{Jawed2018DER, Radha2025Py-DiSMech}. These challenges significantly limit the use of DER models for planning and control. To the best of our knowledge, no prior DER-based framework supports state-space motion planning with direct force inputs. 


\subsection{Trajectory Generation Challenges}
\label{Sect: Related Work-TrajGen}
Trajectory generation in robotic manipulation aims to construct dynamically feasible state--input trajectories $\{\boldsymbol{x}(t), \boldsymbol{u}(t)\}$ that satisfy the system dynamics $\dot{\boldsymbol{x}} = f(\boldsymbol{x}, \boldsymbol{u})$, while achieving desired task objectives defined in task space. These objectives $\boldsymbol{r}(t)$ are typically specified in Euclidean space (e.g., end-effector position), rather than directly in the high-dimensional state space\cite{Matthew2017TrajGen, Sanders2023DynTrajGen, Anthony2022TrajOpt}. 
The first limitation arises in high-fidelity continuum rod formulations: the mathematical complexity of solving the underlying partial or ordinary differential equations with boundaries makes it computationally difficult to ``invert" the model \cite{Till2017Elastic}, which is to determine the exact state trajectory $\boldsymbol{x}(t)$ required to follow a desired task-space path $r(t)$.

The second limitation is in reduced-order models like PCC. While these models can use properties like differential flatness to reconstruct robot shapes and virtual torques from joint angles \cite{Dickson2025Real-Time}, they suffer from an actuation disconnect. Since PCC dynamics are based on equivalent rigid-link approximations, their control input $\boldsymbol{u}$ is a virtual torque rather than a real-world signal such as pneumatic chamber pressure \cite{della_santina_improved_2020, della_santina_controlling_2017}.  In practice, mapping between them requires additional data-driven identification \cite{Dickson2025Safe}, which introduces inherent ambiguity and potential modeling errors in input mapping. Our proposed DDG-based DER formulation addresses these two gaps. Specifically, it provides a consistent mapping between state variables and control inputs, while the proposed actuation-aware formulation enables the construction of dynamically consistent state-input trajectories.

\section{Control-oriented DER Reformulation for Underactuated Soft Robots}
\label{Sect: methodology}

\subsection{DER Dynamics Model and Underactuation}

In this work, we adopt a DDG-based formulation of DER. 
The DER dynamics model adapted from \cite{Choi2024DisMech, Radha2025Py-DiSMech, Choi2021Implicit}, is:
\begin{equation}
    \boldsymbol{M}\ddot{\boldsymbol{q}} 
= \boldsymbol{F}_{\text{int}}(\boldsymbol{q}) 
+ \boldsymbol{F}_{\text{ext}}
\label{Eq:DER_original}
\end{equation}

\noindent where $\boldsymbol{q} \in \mathbb{R}^{4N-1}$ is defined as 
$\boldsymbol{q} = [\boldsymbol{q}_0, \, \phi_0, \,  \dots, \boldsymbol{q}_{N-2}, \, \phi_{N-2}, \, \boldsymbol{q}_{N-1}]^\top$, as Fig.~\ref{Fig:DER Node&Edge Demo}. The DER model describes nodal positions and material frame twists for $N$ discrete point masses that approximate the rod's backbone (Fig.~\ref{Fig:Overview}). The internal elastic force $\boldsymbol{F}_{\text{int}}$ is derived from the total potential energy, which consists of stretching ($E_s$), bending ($E_b$), and twisting ($E_t$) terms:
\[
\begin{aligned}
E_s & = \frac{1}{2} EA \sum_{i=0}^{N-2} (e_i - \bar{e}_i)^2, 
&\quad
E_b & = \frac{1}{2} EI \sum_{i=0}^{N-2} \|\boldsymbol{\kappa}_i - \boldsymbol{\bar{\kappa}}_i\|^2 \\
E_t & = \frac{1}{2} GJ \sum_{i=0}^{N-2} (\tau_i - \bar{\tau}_i)^2,
&\quad
\boldsymbol{F}_{\text{int}}(\boldsymbol{q}) & = -\frac{\partial}{\partial \boldsymbol{q}} \left(E_s + E_b + E_t\right)
\end{aligned}
\]

\noindent where $e_i$, $\boldsymbol{\kappa}_i$, and $\tau_i$ denote the discrete edge length, curvature, and twist, respectively. These quantities are functions of $\boldsymbol{q}$ and are computed following standard DDG-based DER formulations \cite{Jawed2018DER}. For clarity, constant scaling factors are omitted when not affecting the formulation. To apply DER to practical soft robots, the number of nodes $N$ is typically larger than 10, while the robot is actuated through only a small number of inputs (e.g., $\boldsymbol{u} \in \mathbb{R}^2$ for a two-segment design in Fig.~\ref{Fig:Overview}(a)), resulting in a severely \textit{underactuated} system. 
To enable control design under this setting, we reformulate $\boldsymbol{F}_{\text{ext}}$ to make the control input $\boldsymbol{u}$ explicit.

\begin{figure}[t]
\centerline{\includegraphics[width=1\linewidth]{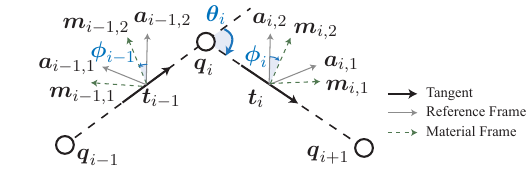}}
\caption{DER uses a physically interpretable state representation based on point masses and edges, and therefore reduces the dynamics to the state $\{\mathbf{q}, \boldsymbol{\phi}\}$, where $\mathbf{q}_i$ denotes positions and $\phi_i$ denotes twist angles. The bending angles $\theta$, material frames $\mathbf{m}$, and reference frames $\mathbf{a}$ can be represented by $\{\mathbf{q}, \boldsymbol{\phi}\}$.}
\label{Fig:DER Node&Edge Demo}
\vspace{-0.5cm}
\end{figure}

\subsection{External Force Decomposition}
As a DER can be regarded as a point-mass system connected by elastic edges, actuation occurs as generalized force vectors applied to those masses, which naturally introduces the notions of direction and magnitude. 
We hypothesize that the force direction is primarily determined by the actuation mechanical design, while the magnitude is induced by the actuator inputs. Following common modeling and calibration assumption \cite{Dickson2025Safe, Gaston2026DERModel}, we treat these two effects as separable, allowing the actuation to be factorized into a configuration-dependent mapping and an input-dependent scaling.
We propose to represent the external force as the sum of a damping term and an actuation-dependent term:
\begin{equation}
 \boldsymbol{F}_{\text{ext}} = \boldsymbol{D}\dot{\boldsymbol{q}} + \boldsymbol{B}(\boldsymbol{q})\boldsymbol{\Lambda}\boldsymbol{u} 
 \label{Eq: DER_revised}
\end{equation}
\noindent where:
\begin{itemize}
    \item $\boldsymbol{D} \in \mathbb{R}^{(4N-1) \times (4N-1)}$: diagonal damping matrix. For simplicity, we assume each element damping independently. D will be identified via data fitting.
    \item $\boldsymbol{B}(\boldsymbol{q}) \in \mathbb{R}^{(4N-1) \times m}$: geometry-dependent actuation map encoding force directions, where $m$ denotes the number of independent control inputs. Its structure will be discussed in Sec.~\ref{Sect: Methodology-ForceAnalysisDER-PCC}.
    \item $\boldsymbol{\Lambda} \in \mathbb{R}^{m \times m}$: scaling matrix mapping actuation inputs to force magnitudes, identified through calibration.
    \item $\boldsymbol{u} \in \mathbb{R}^m$: actuation input, representing control commands (e.g., chamber pressures in pneumatic systems).
\end{itemize}

\noindent This decomposition renders the DER dynamics control-affine with respect to the input $\boldsymbol{u}$, isolating the actuation directions through $\boldsymbol{B}(\boldsymbol{q})$ and the input magnitudes through $\boldsymbol{\Lambda}$. 
As a result, the underactuated input structure becomes explicit.

The key challenge in the proposed formulation lies in constructing the actuation mapping $\boldsymbol{B}(\boldsymbol{q})$. 
Since $\boldsymbol{B}(\boldsymbol{q})$ is inherently determined by the robot design and actuator placement, its structure varies across different systems and is generally nontrivial to derive.
To use \eqref{Eq: DER_revised}, we must derive a formulation of $\boldsymbol{B(q)}$. Here we choose a formulation, motivated by the common PCC assumption that observed in many bending-dominated pneumatic soft robots as Fig.~\ref{Fig:Overview}(a): under many common operating conditions, their deformation typically exhibits a PCC-like curvature profile \cite{dellasantina2018dynamic}. Importantly, we do not impose PCC as a kinematic constraint. Instead, we treat this PCC-like behavior as a nominal deformation manifold that reflects the dominant actuation pattern of the system, and use it to guide the construction of $\boldsymbol{B}(\boldsymbol{q})$. Later work demonstrates that this PCC-informed actuation model combined with DER dynamics improves trajectory generation compared to PCC alone.

\begin{figure}[t]
\centerline{\includegraphics[width=1\linewidth]{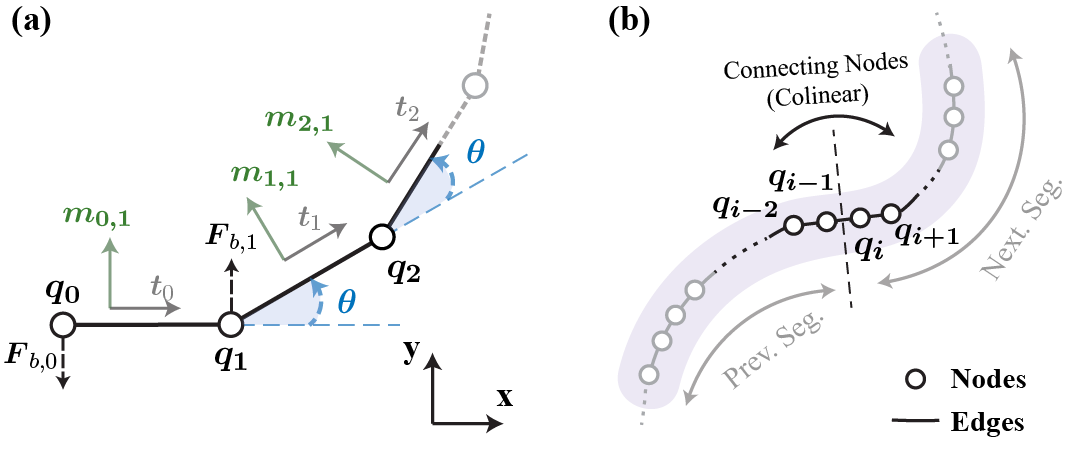}}
\caption{(a) PCC-like DER segment used to illustrate the internal force localization property in Thm.~\ref{Thm: DER-PCC force localization}, proving bending-induced forces concentrated at segment boundaries. 
(b) Illustration of the connecting nodes between adjacent segments in Cor.~\ref{Cor: segment decouple}, proving the colinearity induces the independence for segment-wise force decoupling.}
\label{Fig:DER-PCC Force Localization}
\vspace{-0.3cm}
\end{figure}

\subsection{Force Analysis of PCC-like Rod}
\label{Sect: Methodology-ForceAnalysisDER-PCC}

Under the PCC-informed actuation assumption, we analyze the direction distribution of elastic forces along the rod to motivate the choice of $\boldsymbol{B(q)}$. The following result shows that, under such deformation, the actuation-induced forces are localized at the segment boundaries:

\begin{theorem}[DER--PCC Actuation Force Localization]
Consider a DER segment with nodes $\{\boldsymbol{q}_i\}_{i=0}^{n-1}$ under piecewise-constant curvature. 
Under bending-dominated deformation, the internal elastic bending force vector $\boldsymbol{F}_b$ satisfies
\[
\boldsymbol{F}_{b,i} = \boldsymbol{0}, \qquad \forall i \in  \{2,\ldots,n-3\},
\]
\[
\boldsymbol{F}_{b,0} = -\boldsymbol{F}_{b,1}, \quad \boldsymbol{F}_{b,0}\perp \boldsymbol{t}_{0}
\]
\[
\boldsymbol{F}_{b,n-1} = -\boldsymbol{F}_{b,n-2}, \quad \boldsymbol{F}_{b,n-1}\perp \boldsymbol{t}_{n-2}\]
\noindent forming equal-and-opposite force pairs orthogonal to the corresponding boundary edges.
\label{Thm: DER-PCC force localization}
\end{theorem}

\begin{proof}
As shown in Fig.~\ref{Fig:DER-PCC Force Localization}, under PCC, with no centerline extension or out-of-plane twist:
\[
\boldsymbol{F}_{b,i} = - \frac{\partial E_b}{\partial \boldsymbol{e}_{i-1}} 
+ \frac{\partial E_b}{\partial \boldsymbol{e}_i}.
\]


\noindent For a planar rod with uniform discretization and piecewise-constant curvature, we have the followings for nodes $\{1, 2, \dots, n-2\}:$
\[
\boldsymbol{\bar{\kappa}}_i = (\bar{\kappa}_{i, 1}, \bar{\kappa}_{i, 2}) \equiv (0, 0) \]
\begin{equation}
    \boldsymbol{\kappa}_i(t) = (\kappa_{i, 1}(t), \kappa_{i, 2}(t))  \equiv(\kappa(t), 0)
    \label{Eq: kappa}
\end{equation}
by 
$\boldsymbol{m}_{k, 1} = (\boldsymbol{t}_k \times -\boldsymbol{E}_3)$, 
$\boldsymbol{m}_{k, 2} := \boldsymbol{t}_{k} \times \boldsymbol{m}_{k, 1}$, and 
$\kappa_{i, 1} 
        := \frac{1}{2}(\boldsymbol{m}_{i-1, 2} + \boldsymbol{m}_{i, 2}) \cdot (\kappa \boldsymbol{b})_{i}$, 
$\kappa_{i, 2} 
        := \frac{1}{2}(\boldsymbol{m}_{i-1, 1} + \boldsymbol{m}_{i, 1}) \cdot (\kappa \boldsymbol{b})_{i}$

\noindent Expanding the bending force using DDG formulation gives:
\[
\boldsymbol{F}_{b,i} \propto 
-\kappa_{i-1} \frac{\partial \kappa_{i-1}}{\partial \boldsymbol{e}_{i-1}}
-\kappa_i \frac{\partial \kappa_i}{\partial \boldsymbol{e}_{i-1}}
+\kappa_i \frac{\partial \kappa_i}{\partial \boldsymbol{e}_i}
+\kappa_{i+1} \frac{\partial \kappa_{i+1}}{\partial \boldsymbol{e}_i}.
\]

\noindent This reduces to:
\[
\boldsymbol{F}_{bi} \propto \kappa 
        \cdot\frac{2\kappa(\cos{\theta} + 1) - 4\sin{\theta}}{||\boldsymbol{e}||(1+\cos{\theta})} \cdot 
        (\boldsymbol{t}_{i-1} - \boldsymbol{t}_{i})
\]

\noindent For constant curvature, by \eqref{Eq: kappa}, 
\[
\kappa_{i, 1}(t) 
    = \kappa(t)
    = 2\tan \frac{\theta(t)}{2}
\]

\noindent Therefore, for nodes in $\{2, 3, \dots n-3\}$, we have $\boldsymbol{F}_{bi} = 0$ by $2\kappa(\cos{\theta} + 1) - 4\sin{\theta} = 0$. 
And, at the boundaries, the symmetry is broken due to missing adjacent terms, yielding
\[
\boldsymbol{F}_{b,0} = -\boldsymbol{F}_{b,1} \propto   \frac{\sin \theta}{(1+\cos \theta)^2}\boldsymbol{m}_{0, 1}
\]
\[
\boldsymbol{F}_{b,n-1} = -\boldsymbol{F}_{b,n-2} \propto \frac{\sin \theta}{(1+\cos \theta)^2} \boldsymbol{m}_{n-2, 1}
\]
Thus, the directions are orthogonal to the corresponding boundary edges, and the bending force is localized at the boundary nodes. 

We can assume therefore that actuation forces only occurs at boundary nodes, normal to centerline. 
\end{proof}

This result is consistent with classical \textit{Euler–Bernoulli beam theory}, where constant-curvature deflections occur due to bending moments at boundary conditions only.

Building upon the force localization property established in Thm.~\ref{Thm: DER-PCC force localization}, we extend the analysis to a rod composed of multiple PCC-like segments. The primary challenge lies in the treatment of the interfaces between adjacent segments, where actuation effects may overlap and interact. To address this, we adopt the construction illustrated in Fig.~\ref{Fig:DER-PCC Force Localization}(b). Under this construction, the actuation effects associated with different segments can be introduced independently at the force level, while their coupled deformation behavior is captured by the full DER dynamics. This implies that adjacent PCC-like segments can be connected in a way that preserves force-level independence, providing a structured basis for constructing a segment-wise actuation mapping $\boldsymbol{B}(\boldsymbol{q})$:

\begin{corollary}[Segment-wise Decoupling]
\label{Cor: segment decouple}
Assume the former segment ends at nodes $\{i-2,i-1\}$ and the latter segment starts from nodes $\{i,i+1\}$, where tangent $i$ is defined by nodes $\{i,i+1\}$, as Fig.~\ref{Fig:DER-PCC Force Localization}(b). 
If the tangents $\boldsymbol{t}_{i-2}, \boldsymbol{t}_{i-1}, \boldsymbol{t}_{i}$ are colinear, then the two segments can be regarded as independent in terms of bending-induced forces.
\end{corollary}

\begin{proof}
Following the same argument as in Thm.~\ref{Thm: DER-PCC force localization}, we have
\[
\kappa_{k,1}(t)=\kappa_1(t)=2\tan \frac{\theta_1(t)}{2}, \quad \forall k<i-1, 
\]
\[
\kappa_{k,1}(t)=\kappa_2(t) = 2\tan \frac{\theta_2(t)}{2}, \quad \forall k>i.
\]
And $\kappa_{i-1, 1}(t)=\kappa_{i, 1}(t) \equiv 0$, therefore 
\[
\boldsymbol{F}_{b,i-1} = -\boldsymbol{F}_{b,i-2} \propto \frac{\sin \theta_1}{(1+\cos \theta_1)^2} \boldsymbol{m}_{i-2, 1} 
\]
\[
\boldsymbol{F}_{b,i} = -\boldsymbol{F}_{i,1} \propto   \frac{\sin \theta_2}{(1+\cos \theta_2)^2}\boldsymbol{m}_{0, i}
\]
\noindent Thus, the bending force are independent.
\end{proof}

\noindent Based on Thm.~\ref{Thm: DER-PCC force localization} and Cor.~\ref{Cor: segment decouple}, we can construct a sparse, block-structured $\boldsymbol{B}(\boldsymbol{q})$:

\begin{corollary} [$\boldsymbol{B(q)}$ for DER-PCC]
\label{Cor: construct of B(q)}
Let $N_j$ denotes the total node numbers of the $j^{th}$ limb, $N_0:=0$ and define $s_j := \sum_{k=0}^{j-1} N_k$. Consider a rod composed of $M$ PCC-like segments. 
For each segment $j$, define $\boldsymbol{B}^{j} \in \mathbb{R}^{(4N_j-1)\times 1}$, $\boldsymbol{B}^j_i \in \mathbb{R}^{3}$ as:
\[
\boldsymbol{B}^j = 
[\boldsymbol{B}_{s_j}^j, 0, \boldsymbol{B}_{s_j+1}^j, 0, \dots, \boldsymbol{B}_{s_{j}+N_j-2}^j, 0, \boldsymbol{B}_{s_j+N_j-1}^j,]^\top
\]
\[
\boldsymbol{B}_i^j \neq 0 \text{ iff } i \in \{s_j, s_j+1, s_j+N_j-2, s_j+N_j-1\}
\]

Specifically, for $\boldsymbol{B}_i^j$:
\[
\boldsymbol{B}_{s_j}^j = \mathbb{E}_3 \times \boldsymbol{t}_{s_j}, \quad 
\boldsymbol{B}_{s_j+1}^j = - \boldsymbol{B}_{s_j}^j
\]
\[
\boldsymbol{B}_{s_j+N_j-2}^j = -\mathbb{E}_3 \times \boldsymbol{t}_{s_j+N_j-2}, \quad 
\boldsymbol{B}_{s_j+N_j-1}^j = -\boldsymbol{B}_{s_j+N_j-2}^j  
\]

Therefore, we have:
\begin{equation}
\label{Eq: B(q)}
\boldsymbol{B}(\boldsymbol{q})
:=
\begin{bmatrix}
\boldsymbol{B}^1(\boldsymbol{q}) & \mathbf{0} & \cdots & \mathbf{0}\\
\mathbf{0} & \boldsymbol{B}^2(\boldsymbol{q}) & \cdots & \mathbf{0}\\
\vdots & \vdots & \ddots & \vdots\\
\mathbf{0} & \mathbf{0} & \cdots & \boldsymbol{B}^M(\boldsymbol{q})
\end{bmatrix}
\end{equation}

\end{corollary}
\begin{proof}
From Thm.~\ref{Thm: DER-PCC force localization} and Cor.~\ref{Cor: segment decouple}, internal bending force for $j^{th}$ limb 
$\boldsymbol{F}_b^j$ is supported only at boundary nodes and satisfies 
$\boldsymbol{F}_{b,i}^j \propto \pm (\mathbb{E}_3 \times \boldsymbol{t}_i)$.
Defining $\boldsymbol{B}_i^j := -\boldsymbol{F}_{b,i}^j / \|\boldsymbol{F}_{b,i}^j\|$ yields the result.
\end{proof}

\noindent Crucially, eqn. \eqref{Eq: B(q)} depends only on the geometry of the rod configuration through $\boldsymbol{t}_i$, and is independent of the actuation input $\boldsymbol{u}$ and time explicitly.

\begin{figure*}[t]  
\centerline{\includegraphics[width=1\linewidth]{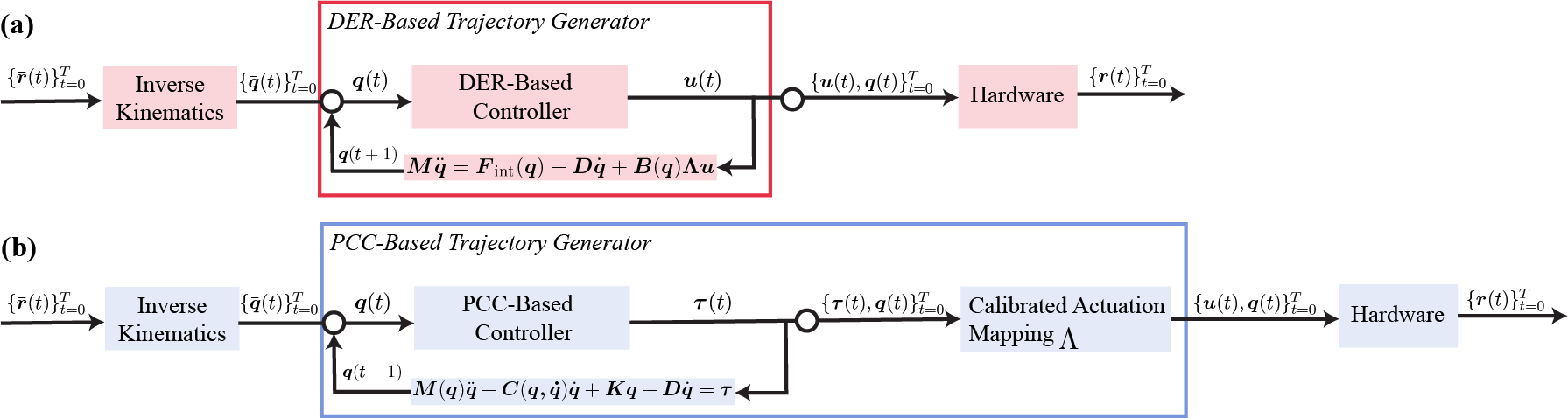}}
\caption{Trajectory generation using (a) the proposed DER formulation and (b) the PCC baseline. Both begin with inverse kinematics to obtain $\{\bar{\boldsymbol{q}}(t)\}_{t=0}^T$ from $\{\bar{\boldsymbol{r}}(t)\}_{t=0}^T$ \cite{Dickson2025Real-Time}. The DER-based approach directly computes actuation inputs $u(t)$ through control-affine dynamics, producing dynamically feasible state–input trajectories $\{\boldsymbol{q}(t), \boldsymbol{u}(t)\}_{t=0}^T$. In contrast, PCC approach computes virtual torques $\boldsymbol{\tau}(t)$ that are subsequently mapped to inputs $\boldsymbol{u}(t)$ via a calibrated mapping $\Lambda$.}
\label{Fig:traj gen diagram}
\end{figure*}

\begin{figure}[b]
\centerline{\includegraphics[width=1\linewidth]{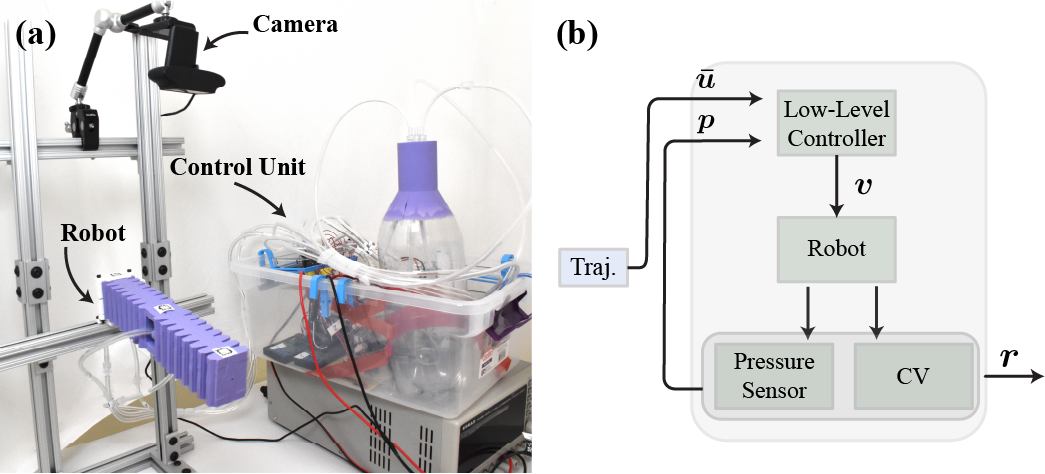}}
\caption{Architecture and hardware overview. Our setup uses (a) computer vision for ground truth measurements of the robot's pose and (b) a cascade control architecture where the high-level DER model generates desired pressure setpoints $\bar{u}$, which are then tracked by a low-level feedback controller with proportional valves and sensors.  
}
\label{Fig: hardware setup}
\vspace{-0.3cm}
\end{figure}

\section{Trajectory Generation}
While the proposed actuation structure in our DER reformulation is inspired by PCC modeling, it is essential to examine how the resulting dynamic model differs from conventional PCC formulations.
Rather than comparing models at a purely kinematic level, we evaluate their dynamic consistency through trajectory generation. Trajectory generation provides a practical and task-oriented benchmark: it requires synthesizing state and actuation sequences that track a task-space end-effector reference $\boldsymbol{r}(t)$ while satisfying the underlying system dynamics. As such, it directly reflects dynamic fidelity and implementation suitability for hardware execution.
To enable the comparison, we construct parallel dynamics-based trajectory generation frameworks for both DER and PCC, as illustrated in Fig.~\ref{Fig:traj gen diagram}. The resulting performance differences are analyzed in Sec.~\ref{Sect: experimental result}, where we demonstrate the improvements of the proposed DER reformulation over conventional PCC modeling. As illustrated in Fig.~\ref{Fig:traj gen diagram}, both frameworks begin with a task-space reference trajectory 
$\{\bar{\boldsymbol{r}}(t)\}_{t=0}^{T}$, which is converted into a reference configuration trajectory 
$\{\bar{\boldsymbol{q}}(t)\}_{t=0}^{T}$ via inverse kinematics \cite{Dickson2025Real-Time}. The distinction lies in how actuation is incorporated into the dynamics.
In the proposed DER reformulation, actuation enters the system explicitly through the control-affine structure. 
As a result, state and input trajectories $\{\boldsymbol{q}(t), \boldsymbol{u}(t)\}_{t=0}^{T}$ are synthesized in a single stage. 
The control input is directly computed from the dynamic equation as
\begin{align}
\boldsymbol{u}(t)
& =
\big(\boldsymbol{B}(\boldsymbol{q}(t))\boldsymbol{\Lambda}\big)^{\dagger}
\Big(
\boldsymbol{M}\ddot{\bar{\boldsymbol{q}}}(t)
-
\boldsymbol{F}_{int}(\bar{\boldsymbol{q}}(t))
+
\boldsymbol{K}_p\big(\bar{\boldsymbol{q}}(t)-\boldsymbol{q}(t)\big) \nonumber \\
& \quad +
\boldsymbol{K}_d\big(\dot{\bar{\boldsymbol{q}}}(t)-\dot{\boldsymbol{q}}(t)\big)
-
\boldsymbol{D}\dot{\boldsymbol{q}}(t)
\Big)
\label{Eq: u for DER}
\end{align} 
where the desired dynamics and feedback terms are resolved directly into actuation commands.
In contrast, the PCC model is formulated in terms of virtual joint torques $\boldsymbol{\tau}(t)$ associated with a dual rigid linkage \cite{della_santina_improved_2020}. 
The practical actuation input is obtained only afterward through a calibrated mapping,
\begin{equation}
\label{Eq: u for PCC}
\boldsymbol{u}(t) = \boldsymbol{\Lambda}^{\dagger}\boldsymbol{\tau}(t),
\end{equation}
resulting in a sequential post-hoc actuation conversion.

For both models, the synthesized state-input trajectories 
$\{\boldsymbol{q}(t), \boldsymbol{u}(t)\}_{t=0}^{T}$ are executed on hardware in an open-loop manner. 
This isolates intrinsic model accuracy from feedback compensation, such that any tracking performance differences arise directly from the underlying dynamic formulation. To evaluate hardware performance, we measure tracking accuracy by comparing the executed end-effector trajectory $\{\boldsymbol{r}(t)\}_{t=0}^{T}$ in Euclidean space, as it reflects the achieved motion in task space rather than internal state consistency.

\begin{figure*}[t]
\centerline{\includegraphics[width=1\linewidth]{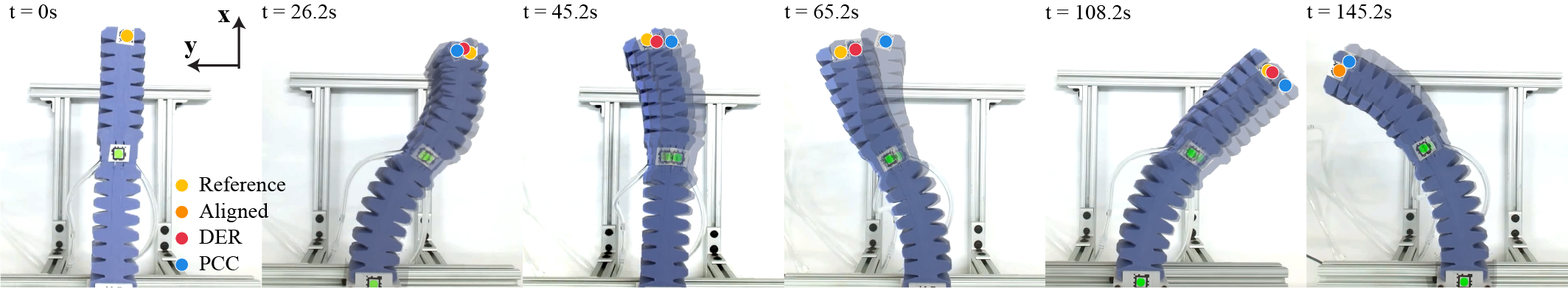}}
\caption{Comparison of trajectories generated using two dynamics models, DER (red) and PCC (blue), against Case 1 (Asynchronous Bending)'s reference poses (yellow). Orange indicates poses where DER aligns with reference (overlap of red and yellow). The figure shows tracking performance over time, demonstrating that DER achieves closer agreement with the reference and overall better accuracy.}
\label{Fig:hardware visualization}
\end{figure*}

\begin{figure*}[t]  
\centerline{\includegraphics[width=1\linewidth]{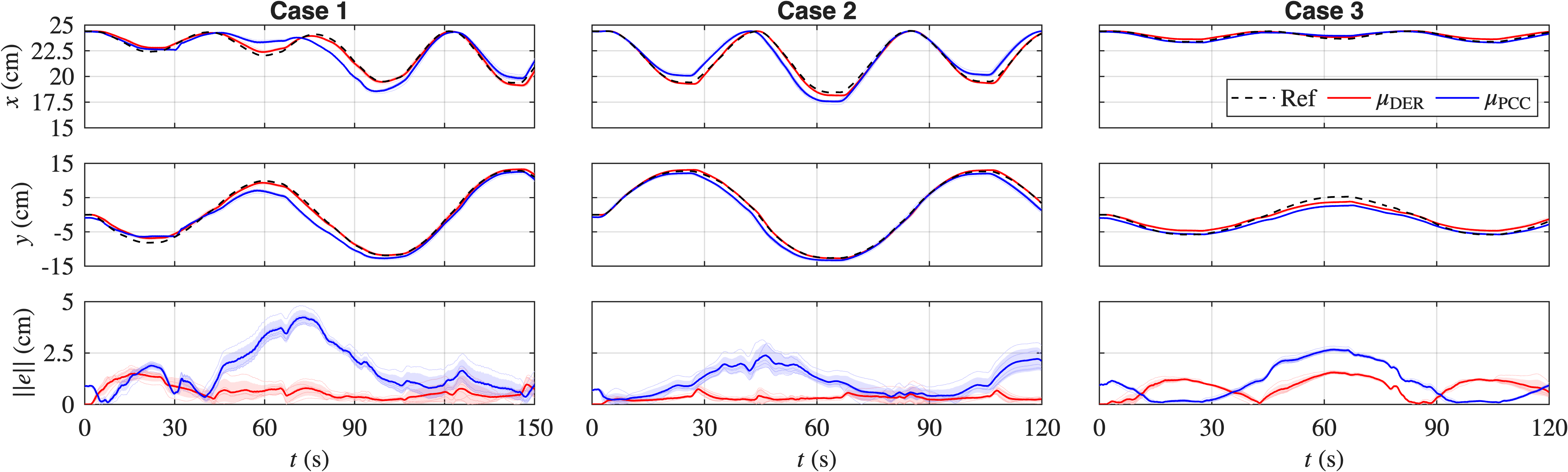}}
\caption{Tracking error ($x, y$) and MSE for three distinct motion profiles (Cases 1–3, Fig. \ref{Fig:hardware visualization}), with five trials per profile. Solid lines and shaded regions represent the mean error $\mu(e)$ and its standard deviation, respectively. Our method consistently outperforms the PCC baseline across all setups.
}
\label{Fig:tracking_performance}
\end{figure*}

\begin{table*}[t]
\centering  
\caption{Tracking performance comparison under different motion scenarios in cm for PCC model \cite{Dickson2025Real-Time} and DER model (this paper)}
\label{tab:tracking_error}
\begin{tabular}{l l c c c}
\hline
\textbf{Scenario} & \textbf{Model} & \textbf{Mean Error E (Total / x / y)} & \textbf{Std (Total / x / y)} & \textbf{Max Error (Total)} \\
\hline

Case 1 
& DER & \textbf{0.64} / \textbf{0.19} / \textbf{0.60} & \textbf{0.11} / \textbf{0.04} / \textbf{0.10} & \textbf{1.90} \\
& PCC & 1.59 / 0.54 / 1.44 & 0.26 / 0.08 / 0.25 & 4.81 \\
\hline
Case 2 
& DER & \textbf{0.29} / \textbf{0.13} / \textbf{0.23} & \textbf{0.05} / \textbf{0.01} / \textbf{0.05} & \textbf{0.98} \\
& PCC & 1.10 / 0.53 / 0.91 & 0.26 / 0.09 / 0.25 & 3.15 \\
\hline
Case 3 
& DER & \textbf{0.84} / 0.12 / \textbf{0.83} & 0.06 / 0.01 / 0.06 & \textbf{1.65} \\
& PCC & 1.03 / 0.12 / 1.01 & \textbf{0.04} / 0.01 / \textbf{0.04} & 2.83 \\

\hline
\end{tabular}
\end{table*}

\section{Experimental Result}
\label{Sect: experimental result}
In this section, we validate our proposed DER dynamics model and trajectory generation approach in hardware. The trajectory generation performance depends on how faithfully our DER model captures the robot’s deformation and actuation behavior. The experiments are designed to evaluate the accuracy of this model in practice. In particular, we compare the proposed DER formulation against a PCC-based baseline to prove that the improved physical representation translates into a more accurate trajectory execution on hardware.

\subsection{Hardware Setup}
The experimental platform we use in the following section consists of a 2-segment pneumatic soft robot with four chambers with planar motion (as shown in Fig.~\ref{Fig: hardware setup}). The robot is driven by differential pressure between paired chambers, which induces planar bending motions. AprilTags \cite{olson2011apriltag} are attached along the robot body and tracked using an overhead camera to estimate the configuration $\boldsymbol{r}$ via computer vision. The control unit operates at 20 Hz and receives the desired pressure difference input $u \in \mathbb{R}^2$. Chamber pressures are regulated by a low-level feedback controller that modulates airflow from a reservoir via proportional valves and pressure sensors. This controller computes real-time valve commands $\boldsymbol{v}$ to track desired pressure differentials between chambers. The prompt pressure response of the low-level system ensures that control inputs $u$ are tracked with a negligible latency of less than $10~\text{ms}$. The robot, measuring approximately $25~\text{cm}$ in length, is fabricated from a silicone elastomer (Smooth-Sil 945, Smooth-On).

\subsection{Experiment Setup}
To systematically examine how each model captures actuation coupling and compliant material interactions, we design a minimal yet representative set of motion scenarios that progressively increase inter-segment coupling complexity. 
The three cases are constructed to isolate and expose different aspects of dynamic interaction between the two soft segments.

We evaluate the proposed pipeline under three motion scenarios: Case 1: \textit{Asynchronous Bending} (Fig.~\ref{Fig:hardware visualization}), Case 2: \textit{Synchronous Bending (Same Direction)}, and Case 3: \textit{Synchronous Bending (Opposite Direction)}, as illustrated in Fig.~\ref{Fig:tracking_performance}. Case 2 represents the simplest configuration, where both segments bend synchronously in the same direction, resulting in minimal antagonistic interaction. 
Case 3 introduces synchronous but opposing bending, creating coordinated yet mechanically competing actuation effects. 
Case 1 further increases complexity by introducing asynchronous bending, intentionally amplifying inter-segment coupling and transient interaction effects. 
Together, these cases form a minimal test set that captures progressively richer coupling behaviors inherent to soft actuation.

For each scenario and each model, five repeated trials are conducted to evaluate consistency. 
Trajectory generation for the DER model is performed using the DisMech MATLAB implementation~\cite{Choi2024DisMech}, 
while the PCC-based trajectory generation is also implemented in MATLAB~\cite{Dickson2025Safe}.

The tracking error is quantified using the $\ell_2$ norm between the measured and reference tip trajectories as $\|\boldsymbol{e}(t)\| = \left\| \boldsymbol{r}(t) - \bar{\boldsymbol{r}}(t) \right\|_2$ at each time step in Fig.~\ref{Fig:tracking_performance}, and 
$E = \frac{1}{T} \int_{0}^{T} \left\| \boldsymbol{r}(t) - \bar{\boldsymbol{r}}(t) \right\|_2 \, dt$
for each trial in Tab.~\ref{tab:tracking_error}.

\subsection{Result}
Fig.~\ref{Fig:tracking_performance} shows the tracking performance of the tip position in terms of $x, \, y,$ and $\ell_2$ norm error $e$ over time for the three motion scenarios. Tab.~\ref{tab:tracking_error} further summarizes the quantitative metrics, including the mean error, standard deviation, and maximum error for each method. Across all three cases, DER consistently outperforms the PCC baseline. In particular, for Case 1 (Asynchronous Bending), DER achieves a mean error of 0.64 cm compared to 1.59 cm for PCC, corresponding to an improvement of approximately $60\%$. Notably, the mean error of DER remains below 1 cm in all scenarios, while PCC exceeds 1 cm in each case. In addition, DER exhibits lower maximum errors across all scenarios, indicating more stable tracking behavior with reduced oscillations, which is also observed in the experimental videos of PCC. 

\subsection{Discussion}
The experimental results consistently demonstrate that the proposed DER formulation outperforms the PCC baseline across all motion scenarios. Although the hardware is designed to exhibit approximately PCC-like behavior, real-world operation inevitably introduces deviations from this idealized assumption. As a result, the true deformation does not necessarily lie within the PCC model class. The improved DER performance observed in Fig.~\ref{Fig:tracking_performance} and Table~\ref{tab:tracking_error} indicates that its ability to capture deviations from ideal PCC assumptions directly contributes to the performance gap between the two models. The larger performance gap observed in Case 1 (asynchronous bending) further supports this interpretation. This motion profile induces more complex inter-segment coordination, increasing sensitivity to modeling assumptions. Since PCC-based models are constrained to a predefined curvature structure, their ability to represent such deviations is inherently limited. In contrast, the proposed DER formulation does not impose a fixed kinematic constraint and instead represents a broader range of deformation behaviors within its state space.

In contrast to high-fidelity PDE-based rod models, which require solving boundary value problems for trajectory inversion as discussed in \ref{Sect: Related Work-TrajGen}, the proposed DER formulation offers a finite-dimensional representation that is directly compatible with control-affine trajectory generation. 
The discrete DER formulation preserves the essential force–deformation relationships while reducing the system to a tractable state-space model that admits explicit actuation inversion \eqref{Eq: u for DER}. This structural property enables direct synthesis of dynamically feasible state–input trajectories without solving PDE-level boundary value problems.

Importantly, the improved tracking performance observed in Table. \ref{tab:tracking_error} directly reflects our  motivation to develop a control-oriented dynamics model that explicitly incorporates actuation and remains compatible with underactuated hardware. The consistently lower mean and maximum errors demonstrate that the generated trajectory pairs $\{\boldsymbol{q}(t), \boldsymbol{u}(t)\}$ are more dynamically consistent with the true system. In particular, the performance gains are most pronounced in scenarios where model–actuation mismatch becomes more significant, indicating that explicit actuation modeling plays a role in trajectory generation.

\section{Conclusion and Future Work}
This work presents a control-oriented reformulation of DER dynamics for soft robots and demonstrates its effectiveness in trajectory generation and hardware execution. By explicitly incorporating actuation into the model structure and exposing a control-affine representation, the proposed formulation enables direct integration of deformation-accurate dynamics into planning pipelines. Unlike prior approaches that rely on indirect or ad-hoc mappings from actuation to deformation variables, the proposed formulation directly relates actuation inputs to generalized forces within the system dynamics. This enables more accurate computation of dynamically feasible trajectory pairs $\{\boldsymbol{q}(t), \boldsymbol{u}(t)\}$, improving execution fidelity without requiring feedback compensation.

Looking forward, when soft robots operate under conditions that deviate from idealized constant-curvature assumptions, such as external loading or contact interactions, we hypothesize the underlying actuation mechanism remains unchanged while the resulting deformation becomes more complex. In such cases, the proposed formulation may better capture these deviations compared to PCC-type models, while maintaining a structured and control-compatible interface for trajectory generation. Future work will focus on extending the proposed formulation beyond the settings considered in this work. In particular, we aim to investigate whether the actuation–force structure remains effective when the system deviates from idealized PCC assumptions. 
Additionally, the explicit relationship between actuation inputs and generalized forces opens new opportunities for sensor-free estimation of external forces, allowing the system to infer interaction forces without dedicated force sensing.

\bibliographystyle{IEEEtran}
\bibliography{references}

\begin{thebibliography}{10}
\providecommand{\url}[1]{#1}
\csname url@samestyle\endcsname
\providecommand{\newblock}{\relax}
\providecommand{\bibinfo}[2]{#2}
\providecommand{\BIBentrySTDinterwordspacing}{\spaceskip=0pt\relax}
\providecommand{\BIBentryALTinterwordstretchfactor}{4}
\providecommand{\BIBentryALTinterwordspacing}{\spaceskip=\fontdimen2\font plus
\BIBentryALTinterwordstretchfactor\fontdimen3\font minus \fontdimen4\font\relax}
\providecommand{\BIBforeignlanguage}[2]{{%
\expandafter\ifx\csname l@#1\endcsname\relax
\typeout{** WARNING: IEEEtran.bst: No hyphenation pattern has been}%
\typeout{** loaded for the language `#1'. Using the pattern for}%
\typeout{** the default language instead.}%
\else
\language=\csname l@#1\endcsname
\fi
#2}}
\providecommand{\BIBdecl}{\relax}
\BIBdecl

\bibitem{vasicSafetyIssuesHumanrobot2013}
M.~Vasic and A.~Billard, ``Safety issues in human-robot interactions,'' in \emph{2013 IEEE International Conference on Robotics and Automation}, 2013, pp. 197--204.

\bibitem{Laschi2016Soft}
C.~Laschi, B.~Mazzolai, and M.~Cianchetti, ``Soft robotics: Technologies and systems pushing the boundaries of robot abilities,'' \emph{Science Robotics}, 2016.

\bibitem{Sanan2011Physical}
S.~Sanan, M.~H. Ornstein, and C.~G. Atkeson, ``Physical human interaction for an inflatable manipulator,'' in \emph{Annual International Conference of the IEEE Engineering in Medicine and Biology Society}, 2011.

\bibitem{Dickson2025Safe}
A.~K. Dickson, J.~C.~P. Garcia, M.~L. Anderson, R.~Jing, S.~Alizadeh-Shabdiz, A.~X. Wang, C.~DeLorey, Z.~J. Patterson, and A.~P. Sabelhaus, ``Safe autonomous environmental contact for soft robots using control barrier functions,'' \emph{IEEE Robotics and Automation Letters}, vol.~10, no.~11, pp. 11\,283--11\,290, 2025.

\bibitem{majidi_2014_soft_robots}
C.~Majidi, ``Soft robotics: A perspective—current trends and prospects for the future,'' \emph{Soft Robotics}, vol.~1, no.~1, pp. 5--11, 2014.

\bibitem{Hsuan2019SoftMedical}
J.-H. Hsiao, J.-Y.~J. Chang, and C.-M. Cheng, ``Soft medical robotics: clinical and biomedical applications, challenges, and future directions,'' \emph{Advanced Robotics}, vol.~33, no.~21, pp. 1099--1111, 2019.

\bibitem{Das2019Review}
A.~Das and M.~Nabi, ``A review on soft robotics: Modeling, control and applications in human-robot interaction,'' in \emph{2019 International Conference on Computing, Communication, and Intelligent Systems (ICCCIS)}, 2019, pp. 306--311.

\bibitem{doroudchi_configuration_2021}
A.~Doroudchi and S.~Berman, ``Configuration {Tracking} for {Soft} {Continuum} {Robotic} {Arms} {Using} {Inverse} {Dynamic} {Control} of a {Cosserat} {Rod} {Model},'' in \emph{{IEEE} {International} {Conference} on {Soft} {Robotics}}, 2021.

\bibitem{renda_dynamics_2024}
F.~Renda, A.~Mathew, and D.~F. Talegon, ``Dynamics and {Control} of {Soft} {Robots} {With} {Implicit} {Strain} {Parametrization},'' \emph{IEEE Robotics and Automation Letters}, 2024.

\bibitem{webster2010design}
R.~J. Webster~III and B.~A. Jones, ``Design and kinematic modeling of constant curvature continuum robots: A review,'' \emph{The International Journal of Robotics Research}, vol.~29, no.~13, pp. 1661--1683, 2010.

\bibitem{katzschmann2019dynamic}
R.~K. Katzschmann, C.~D. Santina, Y.~Toshimitsu, A.~Bicchi, and D.~Rus, ``Dynamic motion control of multi-segment soft robots using piecewise constant curvature matched with an augmented rigid body model,'' in \emph{IEEE International Conference on Soft Robotics (RoboSoft)}, 2019, pp. 454--461.

\bibitem{della_santina_model-based_2020}
C.~Della~Santina, R.~K. Katzschmann, A.~Bicchi, and D.~Rus, ``\BIBforeignlanguage{en}{Model-based dynamic feedback control of a planar soft robot: trajectory tracking and interaction with the environment},'' \emph{\BIBforeignlanguage{en}{The International Journal of Robotics Research}}, vol.~39, no.~4, pp. 490--513, Mar. 2020.

\bibitem{della_santina_improved_2020}
C.~Della~Santina, A.~Bicchi, and D.~Rus, ``On an {Improved} {State} {Parametrization} for {Soft} {Robots} {With} {Piecewise} {Constant} {Curvature} and {Its} {Use} in {Model} {Based} {Control},'' \emph{IEEE Robotics and Automation Letters}, vol.~5, no.~2, pp. 1001--1008, Apr. 2020.

\bibitem{Manti2015Underactuated}
M.~Manti, T.~Hassan, G.~Passetti, N.~d'Elia, M.~Cianchetti, and C.~Laschi, ``An under-actuated and adaptable soft robotic gripper,'' in \emph{4th International Conference on Biomimetic and Biohybrid Systems}, 07 2015.

\bibitem{Thuruthel2020First-Order}
T.~George~Thuruthel, F.~Renda, and F.~Iida, ``First-order dynamic modeling and control of soft robots,'' \emph{Frontiers in Robotics and AI}, vol. Volume 7, 2020.

\bibitem{Naughton2021Elastica}
N.~Naughton, J.~Sun, A.~Tekinalp, T.~Parthasarathy, G.~Chowdhary, and M.~Gazzola, ``Elastica: A compliant mechanics environment for soft robotic control,'' \emph{IEEE Robotics and Automation Letters}, vol.~6, no.~2, pp. 3389--3396, 2021.

\bibitem{Choi2024DisMech}
A.~Choi, R.~Jing, A.~P. Sabelhaus, and M.~K. Jawed, ``Dismech: A discrete differential geometry-based physical simulator for soft robots and structures,'' \emph{IEEE Robotics and Automation Letters}, vol.~9, no.~4, pp. 3483--3490, 2024.

\bibitem{Radha2025Py-DiSMech}
\BIBentryALTinterwordspacing
R.~Lahoti, R.~Chaiyakul, and M.~K. Jawed, ``Py-dismech: A scalable and efficient framework for discrete differential geometry-based modeling and control of soft robots,'' 2025. [Online]. Available: \url{https://arxiv.org/abs/2512.09911}
\BIBentrySTDinterwordspacing

\bibitem{Choi2021Implicit}
A.~Choi, D.~Tong, M.~K. Jawed, and J.~Joo, ``Implicit contact model for discrete elastic rods in knot tying,'' \emph{Journal of Applied Mechanics}, vol.~88, no.~5, p. 051010, 03 2021.

\bibitem{Gaston2026DERModel}
J.~B. Gaston, N.~S. Kumar, E.~J. Barth, and C.~Rucker, ``A 3d discrete elastic rod model and observer for continuum robots,'' \emph{Journal of Mechanisms and Robotics}, vol.~18, no.~3, p. 031002, 01 2026.

\bibitem{Till2017Elastic}
J.~Till and D.~C. Rucker, ``Elastic rod dynamics: Validation of a real-time implicit approach,'' in \emph{2017 IEEE/RSJ International Conference on Intelligent Robots and Systems (IROS)}, 2017, pp. 3013--3019.

\bibitem{Gazzola2018softFilament}
M.~Gazzola, L.~H. Dudte, A.~G. McCormick, and L.~Mahadevan, ``Forward and inverse problems in the mechanics of soft filaments,'' \emph{Royal Society Open Science}, vol.~5, no.~6, p. 171628, 06 2018.

\bibitem{Zhang2019Elastica}
X.~Zhang, F.~Chan, T.~Parthasarathy, and M.~Gazzola, ``Modeling and simulation of complex dynamic musculoskeletal architectures,'' \emph{Nature Communications}, vol.~10, 10 2019.

\bibitem{Bergou2008Discrete}
M.~Bergou, M.~Wardetzky, S.~Robinson, B.~Audoly, and E.~Grinspun, ``Discrete elastic rods,'' in \emph{ACM SIGGRAPH 2008 Papers}.\hskip 1em plus 0.5em minus 0.4em\relax New York, NY, USA: Association for Computing Machinery, 2008.

\bibitem{Goldberg2019OnPD}
N.~N. Goldberg, X.~Huang, C.~Majidi, A.~Novelia, O.~M. O’Reilly, D.~A. Paley, and W.~L. Scott, ``On planar discrete elastic rod models for the locomotion of soft robots,'' \emph{Soft Robotics}, vol.~6, pp. 595 -- 610, 2019.

\bibitem{Liu2022Geometrically}
Y.~Liu, K.~Song, and L.~Meng, ``A geometrically exact discrete elastic rod model based on improved discrete curvature,'' \emph{Computer Methods in Applied Mechanics and Engineering}, vol. 392, p. 114640, 2022.

\bibitem{Huang2022Design}
X.~Huang, Z.~Patterson, A.~Sabelhaus, W.~Huang, K.~Chin, Z.~Ren, K.~Jawed, and C.~Majidi, ``Design and closed‐loop motion planning of an untethered swimming soft robot using 2d discrete elastic rods simulations,'' \emph{Advanced Intelligent Systems}, vol.~4, 09 2022.

\bibitem{huang_dynamic_2020}
W.~Huang, X.~Huang, C.~Majidi, and M.~K. Jawed, ``\BIBforeignlanguage{en}{Dynamic simulation of articulated soft robots},'' \emph{\BIBforeignlanguage{en}{Nature Communications}}, vol.~11, no.~1, p. 2233, May 2020.

\bibitem{Jawed2018DER}
M.~K. Jawed, A.~Novelia, and O.~M. O'Reilly, \emph{A Primer on the Kinematics of Discrete Elastic Rods}, ser. SpringerBriefs in Thermal Engineering and Applied Science.\hskip 1em plus 0.5em minus 0.4em\relax Springer, 2018.

\bibitem{Matthew2017TrajGen}
M.~Kelly, ``An introduction to trajectory optimization: How to do your own direct collocation,'' \emph{SIAM Review}, vol.~59, no.~4, pp. 849--904, 2017.

\bibitem{Sanders2023DynTrajGen}
H.~P. Sanders and M.~D. Killpack, ``Dynamically feasible trajectory generation for soft robots,'' in \emph{IEEE International Conference on Soft Robotics (RoboSoft)}, 2023, pp. 1--8.

\bibitem{Anthony2022TrajOpt}
A.~Wertz, A.~P. Sabelhaus, and C.~Majidi, ``Trajectory optimization for thermally-actuated soft planar robot limbs,'' in \emph{IEEE International Conference on Soft Robotics (RoboSoft)}, 2022.

\bibitem{Dickson2025Real-Time}
A.~Dickson, J.~C.~P. Garcia, R.~Jing, M.~L. Anderson, and A.~P. Sabelhaus, ``Real-time trajectory generation for soft robot manipulators using differential flatness,'' in \emph{2025 IEEE 8th International Conference on Soft Robotics (RoboSoft)}, 2025, pp. 1--7.

\bibitem{della_santina_controlling_2017}
C.~Della~Santina, M.~Bianchi, G.~Grioli, F.~Angelini, M.~Catalano, M.~Garabini, and A.~Bicchi, ``Controlling {Soft} {Robots}: {Balancing} {Feedback} and {Feedforward} {Elements},'' \emph{IEEE Robotics Automation Magazine}, vol.~24, no.~3, pp. 75--83, Sep. 2017.

\bibitem{dellasantina2018dynamic}
C.~Della~Santina, R.~K. Katzschmann, A.~Biechi, and D.~Rus, ``Dynamic control of soft robots interacting with the environment,'' in \emph{IEEE International Conference on Soft Robotics (RoboSoft)}, 2018, pp. 46--53.

\bibitem{olson2011apriltag}
E.~Olson, ``Apriltag: A robust and flexible visual fiducial system,'' in \emph{IEEE International Conference on Robotics and Automation}, 2011.

\end{thebibliography}
\end{document}